\documentclass[a4paper]{article}
\usepackage[utf8x]{inputenc}

\usepackage[a4paper,top=3cm,bottom=2cm,left=3cm,right=3cm,marginparwidth=1.75cm]{geometry}

\usepackage[numbers,sort]{natbib}
\usepackage{amsmath,amsfonts,amsthm}
\usepackage{xspace}
\usepackage{algorithm}
\usepackage{algorithmic}
\usepackage{graphicx}
\usepackage[colorlinks=true, allcolors=blue]{hyperref}
\usepackage{booktabs}
\usepackage{multirow}
\usepackage{subcaption}
\usepackage{pgfplots}
\usepackage{filecontents}
\usepackage{tikz}
\usetikzlibrary{calc}
\usetikzlibrary{mindmap}
\usetikzlibrary{arrows}
\usetikzlibrary{shadows}
\usetikzlibrary{plotmarks}
\usetikzlibrary{backgrounds}
\usetikzlibrary{shapes}
\usetikzlibrary{shapes.symbols}

\newcommand{\twomax}{\textsc{Twomax}\xspace}
\newcommand{\onemax}{\textsc{Onemax}\xspace}
\newcommand{\zeromax}{\textsc{Zeromax}\xspace}

\newcommand{\hamm}{\mathrm{H}}

\newcommand{\poly}[1]{\mathrm{poly}(#1)}
\newcommand{\Hamm}[2]{\hamm\mathord{\left(#1,#2\right)}}

\newcommand{\ones}[1]{\left|#1\right|_1}

\newcommand{\di}{\mathrm{d}}
\newcommand{\dist}[2]{\di\mathord{\left(#1,#2\right)}}

\newcommand{\eaoneone}{\upshape($1+1$)~EA\xspace}
\newcommand{\muea}{{\upshape{}($\mu$+1)~EA}\xspace}

\newcommand{\e}{\mathrm{E}}
\newcommand{\E}[1]{\e\mathord{\left(#1\right)}}
\newcommand{\prob}{\mathrm{Prob}}
\newcommand{\Prob}[1]{\prob\mathord{\left(#1\right)}}
\newcommand{\Bin}[1]{\mathrm{Bin}\mathord{\left(#1\right)}\xspace}

\newcommand{\ie}{i.\,e.,\xspace}

\newcommand{\uar}{u.\,a.\,r.\xspace}

\newcommand{\IFTHENELSE}[3]{\STATE \algorithmicif\ #1\ \algorithmicthen\ #2\ \algorithmicelse\ #3\ \algorithmicendif}

\newcommand{\ignore}[1]{}

\newtheorem{theorem}{Theorem}[section]
\newtheorem{lemma}[theorem]{Lemma}

\pgfplotsset{compat=1.14}
\pgfplotsset{
	enlargelimits=true,
  	grid=major,
  	scale only axis,
}

\usepgfplotslibrary{statistics}
\pgfplotstableread[col sep = tab]{raw_fit_probcrow.txt}\fitrawprobcrow
\pgfplotstableread[col sep = tab]{rtsg.txt}\successrtsg
\pgfplotstableread[col sep = tab]{rtsp.txt}\successrtsp

\title{Runtime Analysis of Probabilistic Crowding and Restricted~Tournament~Selection for Bimodal Optimisation}
\author{Edgar Covantes Osuna and Dirk Sudholt\\\normalsize{}Department of Computer Science\\\normalsize{}University of Sheffield, United Kingdom}

\begin{document}
\maketitle
\begin{abstract}
Many real optimisation problems lead to multimodal domains and so require the identification of multiple optima.
Niching methods have been developed to maintain the population diversity, to investigate many peaks in parallel and to reduce the effect of genetic drift. Using rigorous runtime analysis, we analyse for the first time two well known niching methods: probabilistic crowding and restricted tournament selection (RTS). We incorporate both methods into a \muea on the bimodal function \twomax where the goal is to find two optima at opposite ends of the search space. In probabilistic crowding, the offspring compete with their parents and the survivor is chosen proportionally to its fitness. On \twomax probabilistic crowding fails to find any reasonable solution quality even in exponential time. In RTS the offspring compete against the closest individual amongst $w$ (window size) individuals. We prove that RTS fails if $w$ is too small, leading to exponential times with high probability. However, if $w$ is chosen large enough, it finds both optima for \twomax in time $O(\mu n \log n)$ with high probability. Our theoretical results are accompanied by experimental studies that match the theoretical results and also shed light on parameters not covered by the theoretical results.
\end{abstract}

\section{Introduction}
\label{sec:intro}

Premature convergence is one of the major difficulties in Evolutionary Algorithms (EAs), the population converging to a sub-optimal individual before the fitness landscape is explored properly. Real optimisation problems often lead to multimodal domains and so require the identification of multiple optima, either local or global~\cite{Sareni1998,Singh2006}. In multimodal optimisation problems, there exist many attractors for which finding a global optimum can become a challenge to any optimisation algorithm. A diverse population can deal with these multimodal problems as it can explore several hills in the fitness landscape simultaneously.

One particular way for diversity maintenance are niching methods, based on the mechanics of natural ecosystems~\cite{Shir2012}. A niche can be viewed as a subspace in the environment that can support different types of life. A specie is defined as a group of individuals with similar features capable of interbreeding among themselves but that are unable to breed with individuals outside their group. Species can be defined as similar individuals of a specific niche in terms of similarity metrics. In EAs the term niche is used for the search space domain, and species for the set of individuals with similar characteristics. 

Niching methods have been developed to reduce the effect of genetic drift resulting from the selection operator in standard EAs. This is often done by modifying the selection process of individuals, taking into account not only the value of the fitness function but also the distribution of individuals in the space of genotypes or phenotypes~\cite{Glibovets2013}. Niching methods maintain population diversity and allow the EA to investigate many peaks in parallel, thus avoiding getting trapped in local optima~\cite{Sareni1998}. Many niching techniques have been introduced, including fitness sharing, clearing, probabilistic crowding, deterministic crowding, restricted tournament selection, and many others. The aim is to form and maintain multiple, diverse, final solutions for an exponential to infinite time period with respect to population size, whether these solutions are of identical fitness or of varying fitness~\cite{Shir2012,Crepinsek2013,Glibovets2013,Squillero2016}. Given such a variety of mechanisms to choose from, it is often not clear which mechanism is the best choice for a particular problem.

Most of the analyses and comparisons made between niching methods are assessed by means of empirical investigations using benchmark functions~\cite{Sareni1998,Singh2006}. Theoretical runtime analyses have been performed that rigorously quantify the expected time needed to find one or several global optima~\cite{Friedrich2009,Oliveto2014b,Covantes2017a}. Both approaches are important to understand how these mechanisms impact the EA runtime and if they enhance the search for good individuals. These different expectations imply where EAs and which niching mechanism should be used and, perhaps even more importantly, where they should not be used.

Previous theoretical studies~\cite{Friedrich2009,Oliveto2014b,Covantes2017a} compared the expected running time of different diversity mechanisms when embedded in a simple baseline EA, the \muea. All mechanisms were considered on the well-known bimodal function $\twomax(x):=\max\{\sum_{i=1}^{n}x_i,n-\sum_{i=1}^{n}x_i\}$. \twomax consists of two different symmetric slopes (or branches) \zeromax and \onemax with $0^n$ and $1^n$ as global optima, respectively, and the goal is to evolve a population that contains both optima\footnote{In \cite{Friedrich2009} an additional fitness value for $1^n$ was added to distinguish between a local optimum $0^n$ and a unique global optimum. There the goal was to find the global optimum, and all approaches had a baseline probability of $1/2$ of climbing up the right branch by chance. We use the same approach as~\cite{Oliveto2014b,Covantes2017a}, and consider the original definition of \twomax and the goal of finding both global optima. The discussion and presentation of previous work from~\cite{Friedrich2009} is adapted to our setting. We refer to~\cite{Sudholt2018} for details.}.

\twomax was chosen because it is simply structured, hence facilitating a theoretical analysis, and it is hard for EAs to find both optima as they have the maximum possible Hamming distance. The results allowed for a fair comparison across a wide range diversity mechanisms, revealing that some mechanisms like fitness diversity or avoiding genotype duplicates, perform badly, while other mechanisms like fitness sharing, clearing or deterministic crowding perform surprisingly well (see Table~\ref{tab:divmech} and Section~\ref{sec:previous-work}). 

We contribute to this line of work by studying the performance of two classical diversity mechanisms, \emph{probabilistic crowding} and \emph{restricted tournament selection}. Both methods are well-known techniques as covered in tutorials and surveys for diversity-preserving mechanisms~\cite{Shir2012,Crepinsek2013,Glibovets2013,Squillero2016} and compared in empirical investigations \cite{Sareni1998,Singh2006}. However, we are lacking a good understanding of when and why they perform well and how they compare to diversity mechanisms analysed previously.

In \emph{probabilistic crowding}, offspring compete against their most similar parent and the survivor is chosen with a probability proportional to their fitness~\cite{Mengsheol1999}. The idea is to use a low selection pressure to prevent the loss of niches of lower fitness~\cite{Mengsheol1999}. Probabilistic crowding has been used for multimodal optimisation~\cite{Mengsheol1999,Ballester2004,Mengshoel2008,Mengshoel2014}, in its plain form as well as in variants in which a scaling factor has been introduced into the replacement policy.

\emph{Restricted tournament selection (RTS)} is a modification of the classical tournament selection for multimodal optimisation that exhibits niching capabilities. RTS selects two elements from the population uniformly at random (\uar) to undergo recombination and mutation to produce two new offspring. The offspring compete with their closest individual from $w$ (\emph{window size}) more individuals selected \uar from the population, and the best individual is selected. This form of tournament restricts an entering individual from competing with others too different from it \cite{Harik1995}. RTS has been analysed empirically for the classical comparison between crowding mechanisms for multimodal optimisation as a replacement strategy~\cite{Qu2010,Garcia2012}. Recent applications for engineering problems with multimodal domains include facility layout design~\cite{Garcia2015} and the design of product lines~\cite{Tsafarakis2016} with reported better results compared to the other variants without RTS.

Our contribution is to provide a rigorous theoretical runtime analysis for both mechanisms in the context of the \muea on \twomax, to rigorously assess their performance in comparison to other diversity mechanisms. In addition, our goal is to provide insights into the working principles of these mechanisms to enhance our understanding of their strengths and weaknesses.

For the \muea with probabilistic crowding, we show in Section~\ref{sec:prob_crow} that the mechanism is unable to evolve solutions of significantly higher fitness than that obtained during initialisation (or, equivalently, through random search), even when given exponential time. The reason is that fitness-proportional selection between parent and offspring results in an almost uniform choice as both have very similar fitness, hence fitness-proportional selection degrades to uniform selection for replacement. For the \muea with restricted tournament selection, we show in Section~\ref{sec:rts} that the mechanism succeeds in finding both optima of \twomax in the same way as deterministic crowding, provided that the window size~$w$ is chosen large enough. However, if the window size is too small then it cannot prevent one branch taking over the other, leading to exponential running times with high probability.

Our theoretical results are accompanied by experiments in Section~\ref{sec:exp} that cover a wider range of parameter settings and that show a very good match between our theoretical and empirical results.

\begin{table}
	\centering
    \caption{Overview of runtime analyses for the \muea with different diversity mechanisms on \twomax. Results derived in this paper are shown in bold. The success probability is the probability of finding both optima within (expected) time $O(\mu n \log n)$. Conditions include restrictions on the population size~$\mu$, the sharing/clearing radius~$\sigma$, the niche capacity~$\kappa$, window size~$w$, and $\mu':=\min(\mu,\log{n})$. Results from~\cite{Friedrich2009} are adapted to our definition of \twomax; see~\cite{Sudholt2018} for details.}
    \label{tab:divmech}
    \begin{tabular}{@{}lcl@{}}
      \toprule
      \bf Diversity Mechanism&\bf Success prob. &\bf Conditions\\
      \midrule
      Plain \muea~\cite{Friedrich2009} & $o(1)$& $\mu=o(n/\log{n})$\\
      \midrule
      No Duplicates & & \\
      \quad Genotype~\cite{Friedrich2009} &$o(1)$&$\mu=o(\sqrt[]{n})$\\
      \quad Fitness~\cite{Friedrich2009} &$o(1)$&$\mu=\poly{n}$\\
      \midrule
      Deterministic Crowding~\cite{Friedrich2009} &$1-2^{-\mu+1}$&all $\mu$\\
      \midrule
      Fitness Sharing ($\sigma=n/2$) & & \\
      \quad Population-based~\cite{Friedrich2009}&$1$&$\mu\geq 2$\\
      \quad Individual-based~\cite{Oliveto2014b}&$1$&$\mu\geq 3$\\
      \midrule
      Clearing ($\sigma=n/2$)~\cite{Covantes2017a}&$1$&$\mu\geq\kappa n^2$\\
      \midrule
      \textbf{Probabilistic Crowding}&\boldmath$2^{-\Omega(n)}$& \textbf{all \boldmath$\mu$}\\
      \midrule
      \textbf{RTS}&&\\
      \quad \textbf{Small $w$ sizes}&\boldmath$o(1)$&\boldmath$\mu=o(n^{1/w})$\\
      \quad \textbf{Large $w$ sizes}&\boldmath$1-2^{-\mu'+3}$&\boldmath $w \ge 2.5\mu \ln{n}$\\
      \bottomrule
    \end{tabular}
\end{table}

\subsection{Previous Work}
\label{sec:previous-work}

Table~\ref{tab:divmech} summarises all known results for diversity mechanisms on \twomax. As can be seen, not all mechanisms succeed in finding both optima on \twomax efficiently, that is, in expected time $O(\mu n \log n)$. \citet*{Friedrich2009} showed that the plain \muea and the simple mechanisms like \emph{avoiding genotype or fitness duplicates} are not able prevent the extinction of one branch, ending with the population converging to one optimum, with high probability. \emph{Deterministic crowding} with a sufficiently large population is able to reach both optima with high probability in expected time $O(\mu n \log n)$ \cite[Theorem~4]{Friedrich2009}. A \emph{population-based fitness sharing} approach, constructing the best possible new population amongst parents and offspring, with $\mu\geq 2$ is able to find both optima in expected optimisation time $O(\mu n\log{n})$ \cite[Theorem~5]{Friedrich2009}. The drawback of this approach is that all possible size $\mu$ subsets of this union of size $\mu+\lambda$ (where $\lambda$ is the offspring population size) need to be examined. This is prohibitive for large $\mu$ and $\lambda$. 

\citet*{Oliveto2014b} studied the original \emph{fitness sharing} approach and showed that a population size $\mu=2$ is not sufficient to find both optima in polynomial time; the success probability is only $1/2-\Omega(1)$ \cite[Theorem 1]{Oliveto2014b}. However, with $\mu\geq 3$ fitness sharing again finds both optima in expected time $O(\mu n\log{n})$ \cite[Theorem~3]{Oliveto2014b}. \citet{Covantes2017a} analysed the \emph{clearing} mechanism and showed that it can optimise all functions of unitation---function defined over the number of 1-bits contained in a string---in expected time $O(\mu n\log{n})$ \cite[Theorem~4.4]{Covantes2017a} when the distance function and parameters like the clearing radius $\sigma$, the niche capacity $\kappa$ (how many winners a niche can support) and $\mu$ are chosen appropriately. In the case of large niches, that is, with a clearing radius of $\sigma=n/2$, it is able to find both optima in expected time $O(\mu n\log{n})$ \cite[Theorem~5.6]{Covantes2017a}.

The above works did not consider crossover as recombining individuals from different branches is likely to create poor offspring. We therefore consider a \muea using mutation only.

\section{Probabilistic Crowding}
\label{sec:prob_crow}

We start by presenting and analysing the \muea using probabilistic crowding, defining it in the same fashion as deterministic crowding in \cite{Friedrich2009}.
Recall that in probabilistic crowding, the offspring compete against the most similar parent according to a distance metric and the survivor wins proportionally according to their fitness. Without crossover, this means that the mutant~$y$ competes against its parent~$x$ using fitness-proportional selection. Then the probability of the mutant~$y$ winning is given by $\frac{f(y)}{f(x)+f(y)}$, where $f$ is the fitness function. The resulting \muea is shown in Algorithm~\ref{alg:mueaprobcrow}.

\begin{algorithm}[!ht]
  \begin{algorithmic}[1]
  	\STATE{Let $t:=0$ and initialise $P_0$ with $\mu$ individuals chosen \uar}
    \WHILE{stopping criterion \NOT met}
    	\STATE{Choose $x \in P_t$ \uar}
        \STATE{Create $y$ by flipping each bit in $x$ independently w/\! prob.~$1/n$.\!\!\!}
        \STATE{Choose $r\in [0,1]$ \uar}
        \IFTHENELSE{$r \leq \frac{f(y)}{f(y)+f(x)}$}{$P_{t+1} = P_{t} \setminus \{x\}\cup\{y\}$}{$P_{t+1}=P_{t}$}
    	\STATE{$t:=t+1$.}
    \ENDWHILE
  \end{algorithmic}
  \caption{\muea with probabilistic crowding}
  \label{alg:mueaprobcrow}
\end{algorithm}

There are several related theoretical analyses for fitness-pro\-por\-tional selection for the case of the \onemax function. The \emph{Simple Genetic Algorithm (SGA)} has been analysed with fitness-proportional selection for parent selection in \cite{Neumann2009,Oliveto2014a,Oliveto2015}.

Most relevant to this work is the work by \citet*{Happ2008}, who analysed a variant of the \eaoneone using fitness-proportional selection and showed that it needs exponential time to evolve a fitness of at least $(1+\varepsilon)n/2$ on \onemax with high probability. Their algorithm can be seen as a special case of the \muea with probabilistic crowding for $\mu=1$. Our result is similar to the result in~\cite{Happ2008}, but it holds for arbitrary population sizes~$\mu$ and it applies to both \onemax and \twomax. The proof uses more modern techniques from drift analysis~\cite{Oliveto2011} that were not available to the authors of~\cite{Happ2008}. In the following lemma we define in expectation how the individual accepted for replacement moves away from the individual selected for mutation.

\begin{lemma}
\label{lem:expected-fitness-of-offspring}
Let $x$ be the selected parent, $y$ be the offspring, and $z \in \{x, y\}$ be the individual selected for survival. For $f = \onemax$ if $f(x) \ge (1+\delta)n/2$ for some positive $\delta = \delta(n)$ that may depend on~$n$ then
\[
\E{f(z) - f(x) \mid x} \le -\frac{\delta}{2} + \mathrel{\Theta}\mathrel{\left(\frac{1}{n}\right)}.
\]
The statement also holds for \twomax if $f(x) \ge n/2 + \log n$.
\end{lemma}
\begin{proof}
We first analyse the expected fitness of the mutant~$y$ before survival selection. Compared to its parent~$x$, in expectation at least $(1+\delta)n/2 \cdot 1/n = (1+\delta)/2$ bits flip from 1 to 0, and at most $(1-\delta)n/2 \cdot 1/n = (1-\delta)/2$ bits flip from 0 to 1. Hence
\begin{equation}
\label{eq:expectation-of-offspring}
\E{f(y) - f(x) \mid x} \le (1-\delta)/2 - (1+\delta)/2 = -\delta.
\end{equation}
We now use this inequality to analyse the fitness difference $f(z) - f(x)$ after survival selection. Observe that this difference is 0 in case $z = x$. Hence only generations where $y$ is selected for survival contribute to $\E{f(z) - f(x) \mid x}$. The latter can be written as follows.
\begin{align*}
& \E{f(z) - f(x) \mid x}\\
=\;& \sum_{d=-\infty}^\infty \prob(f(y) - f(x) = d \mid x) \cdot d \cdot \frac{f(y)}{f(x)+f(y)}
\end{align*}
Using that with $d = f(y) - f(x)$,
\begin{align*}
\frac{f(y)}{f(x)+f(y)} = \frac{(f(x)+f(y))/2 + d/2}{f(x)+f(y)} =\;& \frac{1}{2} + \frac{d/2}{f(x)+f(y)}\\
=\;& \frac{1}{2} + \Theta\left(\frac{d}{n}\right),
\end{align*}
we get
\begin{align*}
& \E{f(z) - f(x) \mid x}\\
=\;& \sum_{d=-\infty}^\infty \prob(f(y) - f(x) = d \mid x) \cdot d \cdot \left(\frac{1}{2} + \Theta\left(\frac{d}{n}\right)\right)\\
=\;& \frac{1}{2} \sum_{d=-\infty}^\infty \prob(f(y) - f(x) = d \mid x) \cdot d\\
& \quad + \Theta\left(\frac{1}{n}\right) \sum_{d=-\infty}^\infty \prob(f(y) - f(x) = d \mid x) \cdot d^2.
\end{align*}
The first sum is $\E{f(y) - f(x)}/2$ by definition of the expectation, and we already know from~\eqref{eq:expectation-of-offspring} that $\E{f(y) - f(x)}/2 \le -\delta/2$. The summands in the second sum can be bounded from above using $\prob(f(y) - f(x) = d \mid x) \le 1/(|d|!)$ as it is necessary to flip at least $|d|$ bits, which has probability at most $\binom{n}{|d|} (1/n)^{|d|} \le 1/(|d|!)$. Thus
\begin{align*}
\E{f(z) - f(x) \mid x} \le\;& -\frac{\delta}{2}
  + \Theta\left(\frac{1}{n}\right) \sum_{d=-\infty}^\infty \frac{1}{|d|!} \cdot d^2\\
\le\;& -\frac{\delta}{2}
  + \Theta\left(\frac{1}{n}\right) \cdot 2\sum_{d=1}^\infty \frac{1}{d!} \cdot d^2
= -\frac{\delta}{2}
  + \Theta\left(\frac{1}{n}\right)
\end{align*}
as $\sum_{d=1}^\infty \frac{1}{d!} \cdot d^2 = \sum_{d=1}^\infty \frac{d}{(d-1)!} = \sum_{d=0}^\infty \frac{d+1}{d!} = \sum_{d=0}^\infty \frac{d}{d!} + \sum_{d=0}^\infty \frac{1}{d!} = \sum_{d=1}^\infty \frac{1}{(d-1)!} + \sum_{d=0}^\infty \frac{1}{d!} =
2\sum_{d=0}^\infty \frac{1}{d!} =
2e$.

The statement also holds for \twomax if $f(x) \ge n/2 + \log n$ as the algorithm only ever notices a difference to \onemax in case at least $\log n$ bits flip in one mutation. Since this only occurs with probability at most $1/(\log n)! = n^{-\Omega(\log \log n)}$, and the fitness difference between \onemax and \twomax is at most~$n/2$, this only accounts for an additive error term of $n/2 \cdot n^{-\Omega(\log \log n)} = n^{-\Omega(\log \log n)}$ in the expectation for \onemax, and this error term is absorbed in the $\Theta(1/n)$ term.
\end{proof}

Lemma~\ref{lem:expected-fitness-of-offspring} gives an important lesson. Assume that the survivalist $z$ was chosen uniformly between $x$ and $y$, then we would have
\begin{align*}
& \E{f(z) - f(x) \mid x}\\
=\;& \frac{1}{2} \cdot \E{f(y) - f(x) \mid x}  + \frac{1}{2} \cdot \E{f(x) - f(x) \mid x} = -\frac{\delta}{2}
\end{align*}
using~\eqref{eq:expectation-of-offspring} and $\E{f(x) - f(x) \mid x} = 0$. Lemma~\ref{lem:expected-fitness-of-offspring} states that compared to this setting, a fitness-proportional selection of $z$ only gives a vanishing bias of $\Theta(1/n)$. In other words, Lemma~\ref{lem:expected-fitness-of-offspring} quantifies the observation that in the considered context, fitness-proportional selection is very similar to uniform selection. We now use Lemma~\ref{lem:expected-fitness-of-offspring} to prove a strong negative result on the performance of the \muea with probabilistic crowding. To this end, we will use the negative drift theorem~\cite{Oliveto2011,Oliveto2012Erratum} (also called \emph{simplified drift theorem}). Note that the expected \onemax value of a search point chosen \uar is $n/2$. It is also easy to show that the expected \twomax value of a uniform random search point is $n/2 + \Theta(\sqrt{n})$. These values also represent equilibrium states for sequences of mutations in the absence of selection. The following theorem shows that the \muea with probabilistic crowding does not evolve any solutions of significantly higher fitness than these values, even given exponential time.

\begin{theorem}
\label{the:negprobcrow}
With probability $1-2^{-\Omega(n)}$ the \muea with probabilistic crowding on either $f = \onemax$ or $f = \twomax$ will not have found a search point with fitness at least $(1+\varepsilon)n/2$ in $2^{cn}$ function evaluations, for every population size~$\mu$, every constant $\varepsilon > 0$ and a small enough constant $c > 0$ that may depend on $\varepsilon$.
\end{theorem}
\begin{proof}
We assume that $f = \onemax$ as $\twomax$ can be handled in the same way. We may also assume that $\mu = 2^{o(n)}$ as if $\mu \ge 2^{c'n}$ for any constant~$0 < c' < 1$, the statement follows immediately (for $c := c'$) as the first $2^{c'n}$ search points contain an optimal search point only with probability at most $2 \cdot 2^{-n} \cdot 2^{c'n} = 2^{-\Omega(n)}$ as $c' < 1$. Note that in the absence of crossover, probabilistic crowding evolves $\mu$ independent lineages as any offspring only competes directly with its parent. We show that the probability of any fixed lineage reaching a fitness of at least $(1+\varepsilon)n/2$ in $2^{cn}$ generations is $2^{-\Omega(n)}$. Taking the union bound over all lineages yields that the probability of reaching such a fitness is bounded by $\mu \cdot 2^{-\Omega(n)} = 2^{o(n)} \cdot 2^{-\Omega(n)} = 2^{-\Omega(n)}$, which implies the claim.

Now focus on one lineage. By standard Chernoff bounds (see~\cite{Doerr2018}), the probability of initialising the lineage with an initial search point of fitness at least $(1+\varepsilon/2)n/2$ is $2^{-\Omega(n)}$. If this rare failure event does not happen, the lineage needs to increase an initial fitness from a value at most $(1+\varepsilon/2)n/2$ to a value at least $(1+\varepsilon)n/2$ in order to achieve a fitness of $(1+\varepsilon)n/2$. We apply the negative drift theorem to the fitness of the current individual in our lineage to show that this does not happen in $2^{cn}$ generations with probability $1-2^{-\Omega(n)}$. The interval chosen will be from $a := (1+\varepsilon/2)n/2$ to $b := (1+\varepsilon)n/2$; note that it has length $\varepsilon n/4$.

Let $x$ be the selected parent, $y$ be the offspring, and $z \in \{x, y\}$ be the individual selected for survival. We establish the two conditions of the negative drift theorem. The first condition for search points with fitness at least $a = (1+\varepsilon/2)n/2$ follows from Lemma~\ref{lem:expected-fitness-of-offspring} with $\delta := \varepsilon/2$, yielding a drift of at most $-\varepsilon/4 + \Theta(1/n) = -\Omega(1)$. The second condition follows easily from properties of standard bit mutation: the fitness difference $|f(z) - f(x)|$ is clearly bounded by the number of flipping bits. The probability of flipping $d$ bits in a standard bit mutation is at most $1/(d!) \le 2/2^d$ for all~$d \ge 1$. This proves the second condition when choosing $r := 2$ and $\delta := 1$. Invoking the negative drift theorem yields that the probability of one lineage reaching a search point with fitness at least $(1+\varepsilon)n/2$, starting with a fitness at most $(1+\varepsilon/2)n/2$, in $2^{c'\varepsilon n/2}$ steps, for some constant $c' > 0$, is at most $2^{-\Omega(\varepsilon n/4)} = 2^{-\Omega(n)}$. Choosing $c := c'\varepsilon/4$ completes the proof for \onemax.

The same proof can be used for \twomax with minor modifications: note that if the number of ones is $k \le n/2$, a fitness difference of $d$ can be achieved by increasing or decreasing the number of ones by $d$, provided $k+d \le n/2$ or by creating an offspring with $n-k-d$ ones on the opposite branch. Since $k+d \le n/2 \le n-k-d$, the probability for the latter event is no larger than that of the former. The same holds symmetrically for $k \ge n/2$. Hence all transition probabilities are bounded by twice the previous bound for \onemax and the second condition can be fulfilled by doubling $r$ and choosing $c := c'\varepsilon/8$. Then the result follows as for \onemax.
\end{proof}

\section{Restricted Tournament Selection}
\label{sec:rts}

The \muea with restricted tournament selection (RTS) is defined in a similar way.
Recall that in RTS a new offspring competes with the closest element from $w$ (\emph{window size}) more members selected \uar from the population, and the better individual from this competition is selected. 

In the \muea with RTS, shown in Algorithm~\ref{alg:muearestouse}, an individual~$x$ is selected \uar as a parent and a new individual~$y$ is created in the mutation step. Since we are not considering crossover and only one individual is created, $w$ individuals are selected \uar with replacement and stored in a temporary population~$P^{*}_{t}$. Then in Line~\ref{alg:rts:ln:dis} an individual~$z$ is selected from population~$P^{*}_{t}$ with the minimum distance from~$y$ (ties are broken \uar), and if the individual~$y$ has a fitness at least as good as~$z$, $y$ replaces~$z$.

\begin{algorithm}[!ht]
  \begin{algorithmic}[1]
  	\STATE{Let $t:=0$ and initialise $P_0$ with $\mu$ individuals chosen \uar}
    \WHILE{stopping criterion \NOT met}
    	\STATE{Choose $x \in P_t$ \uar}
        \STATE{Create $y$ by flipping each bit in $x$ independently w/\! prob.~$1/n$.\!\!\!}
        \STATE{Select~$w$ individuals \uar from~$P_{t}$ and store them in~$P^{*}_{t}$.}
        \STATE{Choose $z \in P^{*}_{t}$ such that $\displaystyle\min_{z \in P^{*}_{t}} \dist{y}{z}$.}\label{alg:rts:ln:dis}
        \IFTHENELSE{$f(y) \geq f(z)$}{$P_{t+1} = P_{t} \setminus \{z\}\cup\{y\}$}{$P_{t+1}=P_{t}$}
        \STATE{$t:=t+1$.}
    \ENDWHILE
  \end{algorithmic}
  \caption{\muea with restricted tournament selection}
  \label{alg:muearestouse}
\end{algorithm}

As distance functions~$\di$ we consider genotypic or Hamming distance, defined as the number of bits that have different values in $x$ and $y$: $\dist{x}{y}:=\Hamm{x}{y}:=\sum_{i=0}^{n-1}|x_i - y_i|$, and phenotypic distances as in~\cite{Friedrich2009,Oliveto2014b,Covantes2017a} based on the number of ones: $\dist{x}{y}:=|\ones{x}-\ones{y}|$ where $\ones{x}$ and $\ones{y}$ denote the number of 1-bits in individual $x$ and $y$, respectively.

\subsection{Large Window Sizes Are Effective}
\label{sec:large-w}

Now, let us start with the theoretical analysis for \twomax with a positive result for RTS. The following shows that, if $w$ is chosen very large, the \muea with RTS behaves almost like the \muea with deterministic crowding.
\begin{theorem}
\label{the:positive-result-for-rts}
If $\mu = o(\sqrt{n}/\log n)$ and $w \ge 2.5\mu \ln n$ then the \muea with restricted tournament selection using genotypic or phenotypic distance finds both optima on \twomax in time $O(\mu n \log n)$ with probability at least ${1-2^{-\mu'+3}}$, where $\mu' := \min(\mu, \log n)$.
\end{theorem}
Note that the probability ${1-2^{-\mu'+3}}$ is close to the success probability $1-2^{-\mu+1}$ for deterministic crowding, if $\mu \le \log n$, apart from a constant factor in front of the $2^{-\mu'}$ term. For both, the success rate converges to~1 very quickly for increasing population sizes. For restricted tournament selection our probability bound is capped at $1-2^{-\log n + 3} = 1-8/n$ as there is always a small probability of an unexpected takeover occurring.

In order to prove Theorem~\ref{the:positive-result-for-rts}, we first analyse the probability of initialising a population such that there are individuals on each branch with a safety gap of~$\sigma$ to the border between branches. This safety gap will be used to exclude the possibility that the best individual on one branch creates offspring on the opposite branch.
\begin{lemma}
\label{lem:restriselnich-good-init}
Consider the population of the \muea on \twomax and for some $\mu$ and $\sigma$. The probability of having at least one initial search point with at most $n/2 - \sigma$ ones and one search point with at least $n/2 + \sigma$ ones is at least
\[
1-2\left(\frac{1+2\sigma\cdot\sqrt[]{2/n}}{2}\right)^\mu \ge 1-2^{-\mu+1}(1+o(1))
\]
where the inequality holds if $\sigma \mu = o(\sqrt{n})$.
\end{lemma}
\begin{proof}
Using \cite[Lemma~7]{Doerr2018}, for a random variable with binomial distribution $\Bin{n,1/2}$, for all $z \in [0,n]$ we have
\[
\Prob{X=z}\leq\Prob{X=\lfloor n/2 \rfloor} \le 2^{-n}\cdot\binom{n}{\lfloor n/2 \rfloor} \le \sqrt{2/n}.
\]

So let us start by defining the probability that an individual $x$ is initialised inside the safety gap is at most
\[
p_{\sigma}:=\Prob{n/2-\sigma < \ones{x} < n/2+\sigma}\leq 2\sigma\cdot\sqrt[]{2/n}.
\]
Now let us define the probability that an individual $x$ is initialised on the outer regions with $\ones{x}\leq n/2 - \sigma$ ones ($0^n$ branch) or $\ones{x}\leq n/2 + \sigma$ ones ($1^n$ branch) as $p_{0}$ and $p_{1}$, respectively. Note that both $p_{0}$ and $p_{1}$ are symmetric, and $p_{0} + p_{1}:=1-p_{\sigma}$, and by rewriting we obtain $p_{0}:=\frac{1-p_{\sigma}}{2}$ (the same for $p_{1}$) with its complement being $1-\frac{1-p_{\sigma}}{2}= \frac{1+p_{\sigma}}{2}$.

So the probability of having no individual with at most $n/2-\sigma$ ones is $(1-p_1)^\mu = \left(\frac{1+p_\sigma}{2}\right)^\mu$, and the same holds for having no individual with at least $n/2 + \sigma$ ones. Hence the probability of being initialised as stated in the statement of the lemma is at least
\[
1-2\left(\frac{1+p_{\sigma}}{2}\right)^\mu = 1 - 2^{-\mu+1} \cdot (1+p_\sigma)^\mu.
\]
Plugging in $p_\sigma$ and using the inequality $1+x\leq e^x$ as well as $\sigma\mu = o(\sqrt[]{n})$ we simplify the last term as
\[
 (1+p_\sigma)^\mu \le e^{2\sigma\mu \sqrt[]{2/n}} = e^{o(1)}=\frac{1}{e^{-o(1)}}\leq \frac{1}{1-o(1)} = 1+o(1),
\]
and by plugging all together we have $1-2^{-\mu+1}(1+o(1))$.
\end{proof}

We also show the following time bound, which assumes that the \muea never decreases the best fitness on a considered branch. We will later show in the proof of Theorem~\ref{the:positive-result-for-rts} that this assumption is met with high probability.
\begin{lemma}
\label{lem:time-mu-n-log-n}
Consider one branch of \twomax and a \muea with a replacement selection where the best fitness of all individuals on this branch never decreases. If the \muea is initialised with at least one individual on the branch then the optimum of the branch is found within time $2e\mu n \ln n$ with probability $1-1/n$ and in expectation.
\end{lemma}
\begin{proof}
We apply the multiplicative drift theorem with tail bounds~\cite{DoerrGoldberg2010} to random variables $X_t$ that describe the Hamming distance of the closest individual to the targeted optimum. Note that $X_0 \le n/2$ as we start with an individual on the considered branch and the optimum has been found once $X_t = 0$.

The probability of selecting an individual with Hamming distance $X_t$ is at least $1/\mu$. In order to create a better individual, it is sufficient that one of the $X_t$ differing bits is flipped and the other bits remain unchanged. Each bit flip has a probability of being mutated of $1/n$ and the remaining bits remain unchanged with probability $(1-1/n)^{n-1}$. Hence, the probability of creating an individual with a smaller Hamming distance is bounded as follows:
\[
\prob(X_{t+1} < X_t \mid X_t) \ge \frac{1}{\mu}\cdot\frac{X_t}{n}\cdot\left(1-\frac{1}{n}\right)^{n-1}\geq\frac{X_t}{\mu en}.
\]
This implies
\[
\E{X_{t+1}\mid X_t} \le \left(1-\frac{1}{e\mu n}\right) X_t.
\]
Applying Theorem~1 in~\cite{DoerrGoldberg2010} yields that the time till the optimum is found is at most $e\mu n \cdot (\ln(n/2) + \ln n) \le 2e\mu n \ln n$ with probability at most $1/n$ and in expectation.
\end{proof}

Using these two lemmas, we can now prove Theorem~\ref{the:positive-result-for-rts}.
\begin{proof}[Proof of Theorem~\ref{the:positive-result-for-rts}]
According to Lemma~\ref{lem:restriselnich-good-init}, with probability $1-2^{-\mu+1}(1+o(1))$ the initial population contains at least one search point with at most $n/2-\log n$ ones and at least one search point with at least $n/2+\log n$ ones. We assume in the following that this has happened. The probability of mutation flipping at least $\log n$ bits is at most $1/(\log n)! = n^{-\Omega(\log \log n)}$. Taking the union bound over $O(\mu n \log n)$ steps still gives a superpolynomially small error probability. In the following, we work under the assumption that mutation never flips more than $\log n$ bits.

We call two search points \emph{close} if their genotypic distance is at most $\log n$. Due to our assumption on mutations, every newly created offspring is close to its parent. Note that on \twomax the phenotypic distance of any two search point is bounded from above by the genotypic distance, hence close search points also have a phenotypic distance of at most $\log n$. Note that, whenever the tournament contains a search point that is close to the new offspring, either the offspring or a close search point will be removed. If this always happens, the best individual on any branch cannot be eliminated by an offspring on the opposite branch; recall that initially, the best search points on the two branches have phenotypic distance at least $2\log n$, and this phenotypic distance increases if the best fitness on any branch improves. When genotypic distances are being used, the genotypic distance is always at least $2\log n$.

Since each offspring has at least one close search point (its parent), the probability that the tournament does not contain any close search point is at most $(1-1/\mu)^w \le e^{-w/\mu} = e^{-2.5\ln n} = 1/n^{2.5}$. So long as the best individual on any branch does not get replaced by any individuals on the opposite branch, the conditions of Lemma~\ref{lem:time-mu-n-log-n} are met. Applying Lemma~\ref{lem:time-mu-n-log-n} to both branches, by the union bound the probability of both optima being found in time $2e\mu n \ln n$ is at least $1-2/n$. The probability that in this time a tournament occurs that does not involve a close search point is $O(\mu n \log n) \cdot 1/n^{2.5} = o(1/n)$ as $\mu = o(\sqrt{n}/\log n)$.

All failure probabilities sum up to (assuming $n$ large enough)
\begin{align*}
& \frac{2}{n} + o\left(\frac{1}{n}\right) + 2^{-\mu+1} (1+o(1)) + \frac{O(\mu n \log n)}{n^{-\Omega(\log \log n)}}
\le \frac{4}{n} + 2^{-\mu+2} \le 2^{-\mu'+3}
\end{align*}
where the last inequality follows as $2^{-\mu} \le 2^{-\mu'}$ and $1/n \le 2^{-\mu'}$.
\end{proof}

In Theorem~\ref{the:positive-result-for-rts} we chose $w$ so large that every tournament included the offspring's parent with high probability. Then the \muea behaves like the \muea with deterministic crowding~\cite{Friedrich2009}, leading to similar success probabilities (see Table~\ref{tab:divmech}). 

A success probability around $1-2^{-\mu+1}$ is best possible for many diversity mechanisms as with probability $2^{-\mu+1}$ the whole population is initialised on one branch only (for odd~$n$), and then it is likely that only one optimum is reached. Methods like fitness sharing and clearing obtain success probabilities of~1 by more aggressive methods that can force individuals to travel from one branch to the other by accepting worse search points along the way. The performance of restricted tournament selection (and that of deterministic crowding) is hence best possible amongst all mechanisms that do not allow worse search points to enter the population.

\subsection{Small Window Sizes Can Fail}
\label{sec:small-w}

We now turn our attention to small $w$. If the $w$ is small in comparison to $\mu$, the possibility emerges that the tournament only contains individuals that are far from the offspring. In that case even the closest individual in the tournament will be dissimilar to the offspring, resulting in a competition between individuals from different ``niches'' (\ie sets of similar individuals). The following theorem and its proof show that this may result in one branch taking over the other branch, even when the branch to get extinct is very close to a global optimum. The resulting expected optimisation time is exponential.

\begin{theorem}
\label{the:badresrts}
Let $\mu \le n/8$. The probability that the \muea with restricted tournament selection with $w\ge 2$ and either genotypic or phenotypic distances finds both optima on \twomax in time $n^{n-1}$ is at most $O(\mu^w/n)$. If $\mu \le \varepsilon n^{1/w}$ for a sufficiently small constant~$\varepsilon > 0$ then the expected time for finding both optima is $\Omega(n^n)$.
\end{theorem}
Note that the probability of finding both optima in $n^{n-1}$ generations is $o(1)$ if $w = O(1)$ and $\mu$ grows slower than the polynomial $n^{1/w}$.
It also holds if
$w \le c(\ln n)/\ln \ln n$ for some constant $0 < c < 1$ and $\mu = O(\log n)$ as then $n^{1/w} = e^{(\ln n)/w} \ge e^{(\ln \ln n)/c} = (\ln n)^{1/c} = \omega(\log n)$, which shows $\mu^w/n = o(1)$.
\begin{proof}[Proof of Theorem~\ref{the:badresrts}]
The analysis follows the proof of~\cite[Theorem~1]{Friedrich2009}. We assume that the initial population contains at most one global optimum as the probability of both optima being found during initialization is at most $\mu\cdot 2^{-n} = O(\mu^w/n)$.

We consider the first point of time at which the first optimum is being bound. Without loss of generality, let us assume that this is $0^n$. Then we show that with high probability copies of $0^n$ take over the population before the other optimum~$1^n$ is found.

Let~$i$ be the number of copies of the~$0^n$ individuals in the population, then a good event~$G_i$ (good in a sense of leading towards extinction as we are aiming at a negative result) is to increase this number from~$i$ to~$i+1$. For this it is just necessary to create copies of one of the~$i$ individuals. For~$n\geq 2$ we have~$\Prob{G_i}\geq\frac{i}{\mu}\cdot\left(1-\frac{1}{n}\right)^n\cdot\left(\frac{\mu-i}{\mu}\right)^w\geq\frac{i}{4\mu}\cdot\left(\frac{\mu-i}{\mu}\right)^w$ since it suffices to select one out of~$i$ individuals and to create a copy of the selected individual, and to select~$w$ times individuals from the remaining~$\mu-i$ individuals. On the other hand, a bad event~$B_i$ is to create an~$1^n$ individual in one generation. This probability is clearly bounded by  $\Prob{B_i}\leq\frac{1}{n}$ as every individual with at least one zero bit has to flip said bit to create~$1^n$. Together, the probability that the good event~$G_i$ happens before the bad event~$B_i$ is
\begin{align*}
 \Prob{G_i\mid G_i\cup B_i}\geq\;& \frac{\Prob{G_i}}{\Prob{G_i}+\Prob{B_i}}
 \geq\frac{\frac{i}{4\mu}\cdot\left(\frac{\mu-i}{\mu}\right)^w}{\frac{i}{4\mu}\cdot\left(\frac{\mu-i}{\mu}\right)^w+\frac{1}{n}} \\
 =\;& 1-\frac{\frac{1}{n}}{\frac{i}{4\mu}\cdot\left(\frac{\mu-i}{\mu}\right)^w+\frac{1}{n}}
 \geq 1-\frac{4\mu}{in \cdot ((\mu-i)/\mu)^w}.
\end{align*}
The probability that the $i$ individuals take over the population before $1^n$ is found is therefore at least
\[
\prod_{i=1}^{\mu}\Prob{G_i\mid G_i\cup B_i}\geq\prod_{i=1}^{\mu}\bigg(1-\frac{4\mu}{in \cdot ((\mu-i)/\mu)^w}\bigg).
\]

Using $\frac{4\mu}{n}\leq \frac{1}{2}$ and $1-x\geq e^{-2x}$ for $x\leq \frac{1}{2}$, we obtain
\begin{align*}
\prod_{i=1}^{\mu}\left(1-\frac{4\mu}{in \cdot ((\mu-i)/\mu)^w}\right)&\geq\prod_{i=1}^{\mu}\exp\left(-\frac{8\mu}{in \cdot ((\mu-i)/\mu)^w}\right)\\
&=\exp\left(-\frac{8\mu}{n}\cdot\sum_{i=1}^{\mu-1}\frac{1}{i \cdot ((\mu-i)/\mu)^w}\right)\\
&=\exp\left(-\frac{8\mu}{n}\cdot\mu^{w}\sum_{i=1}^{\mu-1}\frac{1}{i\cdot\left(\mu-i\right)^w}\right).
\end{align*}

Note that the summands are non-increasing with $w$. So the worst case is having the smallest possible $w\geq 2$, so we can bound this sum from above in the following way:
\begin{align*}
\sum_{i=1}^{\mu-1} \frac{1}{i \cdot (\mu-i)^2}
\le\;& \sum_{i=1}^{\lfloor \mu/2 \rfloor} \frac{1}{i \cdot (\mu-i)^2} + \sum_{i=\lceil \mu/2 \rceil}^{\mu-1} \frac{1}{i \cdot (\mu-i)^2}\\
\le\;& \sum_{i=1}^{\lfloor \mu/2 \rfloor} \frac{1}{i \cdot (\mu/2)^2} + \sum_{i=\lceil \mu/2 \rceil}^{\mu-1} \frac{1}{\mu/2 \cdot (\mu-i)^2}\\
\le\;& \frac{4}{\mu^2} \sum_{i=1}^{\lfloor \mu/2 \rfloor} \frac{1}{i} + \frac{2}{\mu} \sum_{i=1}^{\infty} \frac{1}{i^2} = O\left(\frac{1}{\mu}\right)
\end{align*}
as $\sum_{i=1}^{\lfloor \mu/2 \rfloor} \frac{1}{i} = O(\log \mu)$ and $\sum_{i=1}^\infty 1/i^2 = \pi^2/6$.
Together we have
\begin{align*}
\prod_{i=1}^{\mu}\Prob{G_i\mid G_i\cup B_i}&\geq\exp\left(-\frac{8\mu}{n}\cdot \mu^w \cdot O\left(\frac{1}{\mu}\right)\right)
\geq 1-O\left(\frac{\mu^w}{n}\right).
\end{align*}

Once the population consists only of copies of $0^n$, a mutation has to flip all $n$ bits to find the $1^n$ optimum. This event has probability $n^{-n}$ and, by the union bound, the probability of this happening in a phase consisting of $n^{n-1}$ generations is at most $\frac{1}{n} = O(\mu^w/n)$. The sum of all failure probabilities is $O(\mu^w/n)$, which proves the first claim. For the second claim, observe that the conditional expected optimization time is $n^n$ once the population has collapsed to copies of $0^n$ individuals. As this situation occurs with probability at least $1-O(\mu^w/n) = \Omega(1)$ if the constant $\varepsilon$ in $\mu \le \varepsilon n^{1/w}$ is sufficiently small, the unconditional expected optimization time is $\Omega(n^n)$.
\end{proof}

\section{Experiments}
\label{sec:exp}

We provide an experimental analysis as well in order to see how closely the theory matches the empirical performance for reasonable problem sizes, and to investigate a wider range of parameters, where the theoretical results are not applicable. Our analysis is focused on the \muea with probabilistic crowding and RTS for the \twomax function. We consider exponentially increasing population sizes~$\mu=\{2,4,8,\ldots,1024\}$ for a problem size $n=100$ and for $100$ runs.

Since we are interested in proving how good/bad all algorithms are, we define the following outcomes and stopping criteria for each run. \emph{Success}, both optima of \twomax have been reached, \ie the run is stopped if the population contains both $0^n$ and $1^n$ in the population. \emph{Failure}, once the run has reached $10\mu n\ln n$ generations and the population does not contain both optima. By Lemma~\ref{lem:time-mu-n-log-n}, this time period is long enough to allow any reasonable \muea variant to find one or two global optima with high probability (unless the best fitness on a branch drops frequently). We report the mean of successes and failures for the $100$ runs.

For probabilistic crowding (Algorithm~\ref{alg:mueaprobcrow}), and as proved in Theorem~\ref{the:negprobcrow}, for all $\mu$ sizes, the method is not able to optimise \twomax. In all runs the algorithm failed to reach even one optimum, let alone reaching both. Since the algorithm is not able to find any optimum of \twomax, we ran additional experiments for $n=\{32,64,128,\ldots,16384\}$ and population size $\mu=32$ to observe how far the best lineages evolve from $n/2$ and/or how close the best individuals get to reach an optimum. In Figure~\ref{fig:probcrowbp}, we show the best individuals obtained in each of the $100$ runs and its variance. As soon as $n$ increases, the best fitness in the population starts to concentrate around $n/2$ and reaching a fitness of $(1+\varepsilon) n/2$ becomes very difficult for all constants $\varepsilon > 0$ as $n$ grows. Even the best outliers start to get closer and closer to the average of the population.

\begin{figure}[ht]
    \centering
    \resizebox{.60\linewidth}{!}{
    \begin{tikzpicture}
    	\begin{axis}[
          width=7cm,
          height=6cm,
          boxplot/draw direction=y,
          xlabel=$n$ (logscale),
          xtick={1,2,3,4,5,6,7,8,9,10},
          xticklabels={32,64,128,256,512,1024,2048,4096,8192,16384},
          ytick={0.5,0.6,0.7,0.8,0.9,1},
          ymin=0.5, ymax=1,
          ylabel=$\frac{\twomax(y)}{n}$,
          x tick label style={rotate=45, anchor=north east, inner sep=0mm},
        ]
          \addplot[boxplot,mark=*,color=blue] table[y index=0, row sep=newline]{\fitrawprobcrow};
          \addplot[boxplot,mark=*,color=blue] table[y index=1, row sep=newline]{\fitrawprobcrow};
          \addplot[boxplot,mark=*,color=blue] table[y index=2, row sep=newline]{\fitrawprobcrow};
          \addplot[boxplot,mark=*,color=blue] table[y index=3, row sep=newline]{\fitrawprobcrow};
          \addplot[boxplot,mark=*,color=blue] table[y index=4, row sep=newline]{\fitrawprobcrow};
          \addplot[boxplot,mark=*,color=blue] table[y index=5, row sep=newline]{\fitrawprobcrow};
          \addplot[boxplot,mark=*,color=blue] table[y index=6, row sep=newline]{\fitrawprobcrow};
          \addplot[boxplot,mark=*,color=blue] table[y index=7, row sep=newline]{\fitrawprobcrow};
          \addplot[boxplot,mark=*,color=blue] table[y index=8, row sep=newline]{\fitrawprobcrow};
          \addplot[boxplot,mark=*,color=blue] table[y index=9, row sep=newline]{\fitrawprobcrow};
    	\end{axis}
    \end{tikzpicture}
    }
    \caption{The normalised best fitness $\twomax/n$ reached among $100$ runs at the time both optima were found or the $t=10\mu n \ln n$ generations have been reached on \twomax for $n=\{32,64,128,\ldots,16384\}$ by the \muea with probabilistic crowding with $\mu=32$.}
    \label{fig:probcrowbp}
\end{figure}
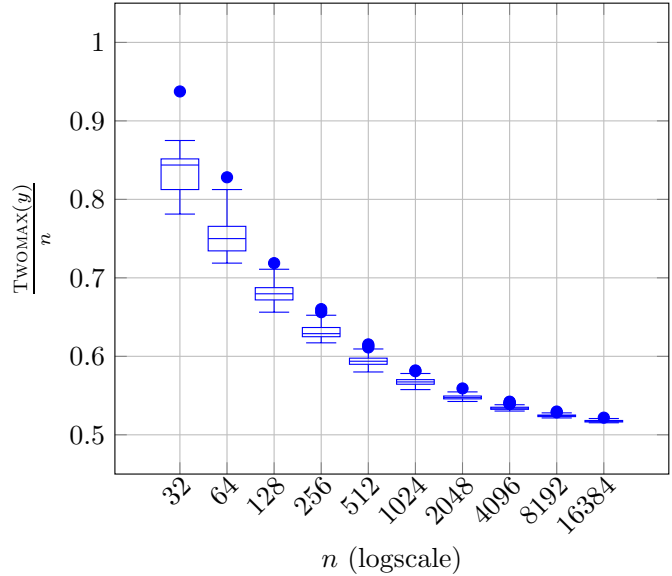

In the case of RTS (Algorithm~\ref{alg:muearestouse}), we ran experiments for $w=\{1,2,4,8,\ldots,1024\}$, however we only plot results up to~$w=128$ as the results for large~$w$ were very similar. Figure~\ref{fig:successrts} shows that for small values of $w$ and $\mu$ the algorithm is not able to maintain individuals on both branches of \twomax for a long period of time, as predicted by Theorem~\ref{the:badresrts}. It is only when the population size is set to $\mu=1024$ (where Theorem~\ref{the:badresrts} does not apply any more since $\mu > n/8$) the algorithm is able to maintain individuals on both branches before the takeover happens. When setting, for example, $w\geq 8$ and $\mu\geq 32$ the algorithm was able to find both optima with both genotypic and phenotypic distances. It is possible to observe a trade-off between $w$ and $\mu$: larger~$w$ allow for a smaller population size $\mu$ to be used. Such a trade-off was also indicated by the probability bound $O(\mu^w/n)$ from Theorem~\ref{the:badresrts}. 
Our experiments show that RTS works well for much smaller window sizes than those required in Theorem~\ref{the:positive-result-for-rts}. As a final remark, the method seems to behave fairly similarly with respect to both distance functions.

\begin{figure}[ht]
    \begin{subfigure}[t]{0.5\textwidth}
        \centering
        \resizebox{\linewidth}{!}{
        	\begin{tikzpicture}	
    			\begin{axis}[
                    width=7cm,
        			height=6cm,
                    xlabel=$\mu$ (logscale),
        			xtick={1,2,3,4,5,6,7,8,9,10},
                    xticklabels={2,4,8,16,32,64,128,256,512,1024},
        			every axis y label/.style={rotate=0, black, at={(-0.13,0.5)},},
                    legend columns=4,
                    legend style={at={(0.5,-0.16)},anchor=north},
            		legend cell align=left,
					]
						\addplot table[x=mu, y=w1]{\successrtsg};
                    	\addplot table[x=mu, y=w2]{\successrtsg};
                        \addplot table[x=mu, y=w4]{\successrtsg};
                        \addplot table[x=mu, y=w8]{\successrtsg};
                        \addplot table[x=mu, y=w16]{\successrtsg};
                        \addplot table[x=mu, y=w32]{\successrtsg};
                        \addplot table[x=mu, y=w64]{\successrtsg};
                        \addplot table[x=mu, y=w128]{\successrtsg};
                        \legend{$w=1$, $w=2$, $w=4$, $w=8$, $w=16$, $w=32$, $w=64$, $w=128$, $w=256$, $w=512$, $w=1024$};
    			\end{axis}
		\end{tikzpicture}
		}
        \caption{Genotypic}
        \label{fig:successrtsg}
    \end{subfigure}
    \begin{subfigure}[t]{0.5\textwidth}
        \centering
        \resizebox{\linewidth}{!}{
        	\begin{tikzpicture}	
        	\begin{axis}[
        		width=7cm,
        		height=6cm,
        		xlabel=$\mu$ (logscale),
        		xtick={1,2,3,4,5,6,7,8,9,10},
                xticklabels={2,4,8,16,32,64,128,256,512,1024},
        		every axis y label/.style={rotate=0, black, at={(-0.13,0.5)},},
                legend columns=4,
                legend style={at={(0.5,-0.16)},anchor=north},
            	legend cell align=left,
        	]
        		\addplot table[x=mu, y=w1]{\successrtsp};
                \addplot table[x=mu, y=w2]{\successrtsp};
                \addplot table[x=mu, y=w4]{\successrtsp};
                \addplot table[x=mu, y=w8]{\successrtsp};
                \addplot table[x=mu, y=w16]{\successrtsp};
                \addplot table[x=mu, y=w32]{\successrtsp};
                \addplot table[x=mu, y=w64]{\successrtsp};
                \addplot table[x=mu, y=w128]{\successrtsp};
                \legend{$w=1$, $w=2$, $w=4$, $w=8$, $w=16$, $w=32$, $w=64$, $w=128$, $w=256$, $w=512$, $w=1024$};
        	\end{axis}
		\end{tikzpicture}  	
		}
        \caption{Phenotypic}
        \label{fig:successrtsp}
    \end{subfigure}
    \caption{The number of successful runs measured among $100$ runs at the time both optima were found on \twomax or $t=10\mu n \ln n$ generations have been reached for $n=100$ with the \muea with restricted tournament selection with $\mu=\{2,4,8,\ldots,1024\}$, $w=\{1,2,4,8,\ldots,128\}$, genotypic and phenotypic distance.}
    \label{fig:successrts}
\end{figure}
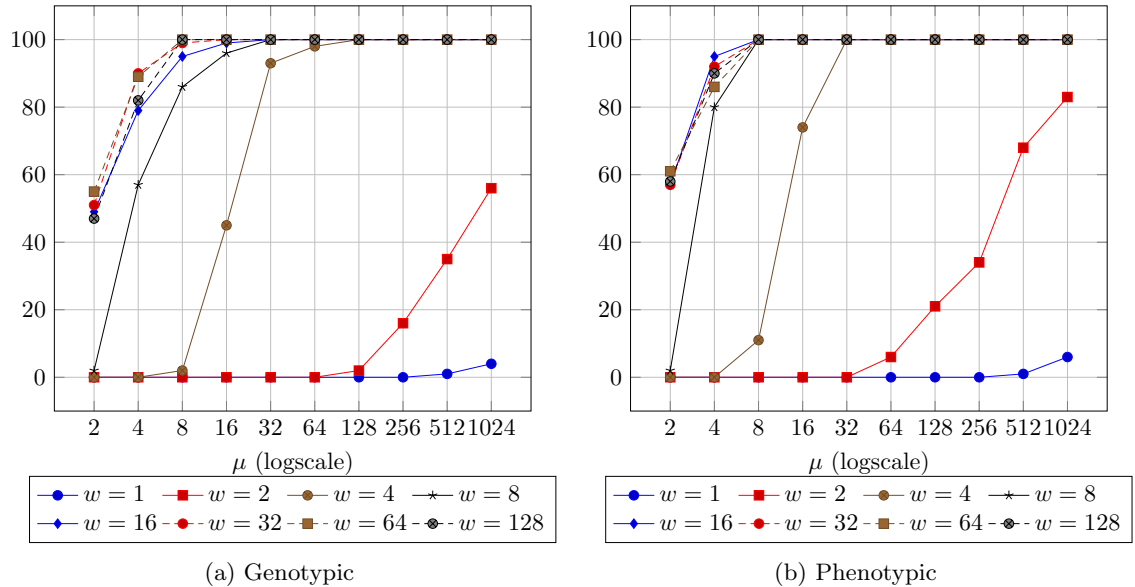

\section{Conclusions}
\label{sec:con}
We have examined theoretically and empirically the behaviour of two different niching mechanisms embedded into a simple \muea on \twomax. We rigorously proved that probabilistic crowding fails miserably; it is not even able to evolve search points that are significantly better than those found by random search, even when given exponential time. The reason is that fitness-proportional selection for survival selection works very similar to uniform selection, and then the algorithm performs an almost blind search on $\mu$ independent lineages. 

Our results highlight the importance of scaling the fitness, as done in~\cite{Ballester2004,Mengshoel2008,Mengshoel2014}. An open question is whether fitness scaling would enable probabilistic crowding to find both optima on \twomax, and if so, how much the fitness needs to be scaled. A proof where fitness scaling has helped for a variant of the Simple GA on \onemax was given in~\cite{Neumann2009}. We are confident that the proof arguments used here can also be used to analyse more advanced versions of crowding~\cite{Galan2010,Mengshoel2014}.

The performance of restricted tournament selection seems to vary a lot with the parameters involved. We have shown that if $\mu$ and $w$ are set too small, one subpopulation may get extinct. But if $w$ is large enough then RTS behaves similarly to deterministic crowding. For both the probability of finding both optima is close to $1-2^{-\mu+1}$, hence converging to~1 very quickly as $\mu$ grows. It still an open problem to theoretically analyse the population dynamics of RTS for intermediate values for $w$. Our experiments show that RTS can optimise \twomax for smaller $w$ than the one required in Theorem~\ref{the:positive-result-for-rts}.

\section*{Acknowledgements}
The authors would like to thank the anonymous reviewers from GECCO '18 for their many valuable suggestions. We also thank the Consejo Nacional de Ciencia y Tecnolog\'{i}a --- CONACYT (the Mexican National Council for Science and Technology) for the financial support under the grant no.\ 409151 and registration no.\ 264342. 


\bibliographystyle{abbrvnat} 
\bibliography{sigproc} 

\end{document}